\documentclass{article}

\usepackage[preprint]{log_2026}

\usepackage{amsmath,amssymb,amsfonts,amsthm}
\usepackage{booktabs}
\usepackage{multirow}
\usepackage{graphicx}
\usepackage{microtype}
\usepackage{enumitem}
\usepackage{xcolor}
\usepackage{tikz}
\usetikzlibrary{arrows.meta,positioning,fit,calc,shapes.geometric}
\usepackage[numbers,compress,sort]{natbib}
\usepackage{url}

\newtheorem{theorem}{Theorem}[section]
\newtheorem{proposition}[theorem]{Proposition}
\newtheorem{lemma}[theorem]{Lemma}
\newtheorem{corollary}[theorem]{Corollary}
\newtheorem{definition}[theorem]{Definition}
\theoremstyle{remark}
\newtheorem{remark}[theorem]{Remark}

\newcommand{\R}{\mathbb{R}}
\newcommand{\calS}{\mathcal{S}}
\newcommand{\calT}{\mathcal{T}}

\newcommand{\softmax}{\operatorname{softmax}}

\newcommand{\MLP}{\operatorname{MLP}}
\newcommand{\ELU}{\operatorname{ELU}}

\definecolor{sourceblue}{RGB}{70,130,180}
\definecolor{targetorange}{RGB}{230,135,60}
\definecolor{slotgreen}{RGB}{70,155,105}
\definecolor{anchorpurple}{RGB}{130,90,170}
\definecolor{lightgray}{RGB}{235,235,235}
\definecolor{accentred}{RGB}{200,70,60}

\title[Escaping Oversquashing]{Escaping Oversquashing: Addressable and Support-Aware Global Memory for Message Passing Networks}

\author[F\'elix Marcoccia]{%
F\'elix Marcoccia\\
}

\begin{document}
\maketitle

\begin{abstract}
Virtual nodes are a natural tool against oversquashing: they replace long
message-passing paths by a two-hop global route. But when many nodes share one
global state, that shortcut can become a bottleneck itself. We study two
properties of this global memory. First, \emph{addressability}: under
constant-margin address codes and a nonlinearity that amplifies this margin,
multiplicative write/read maps provide $M$ selectable memory rows with only
$O(\log M)$ address-code dimensions. Cross-attention slots and a constrained
$\ELU+1$ bilinear memory both satisfy these conditions. Second,
\emph{support awareness}: normalized cross-attention has no self-key for a
latent query to use as a reference. A learned private anchor supplies this
reference, keeps the read bounded, and exposes the strength of the matching
source mass. We demonstrate the merits of such properties on several instances of Two-Radius and Tree-NeighborsMatch: both addressable
realizations solve the controlled tasks through depth $5$, where pooled VNs of comparable or larger size
reach about $10.6\%$. 
\end{abstract}

\section{Introduction}
\label{sec:intro}

Message-passing neural networks (MPNNs) keep computation local and sparse
\citep{gilmer2017neural}. The price is that a node only sees far-away
information after it has travelled through many intermediate states. When many
signals must share the same few states, they get compressed and partly lost.
This is \emph{oversquashing}
\citep{alon2021bottleneck,topping2022understanding,di2023over}.

A virtual node (VN) is a simple way to shorten these paths: one extra state is
connected to every node, so any two nodes can communicate through a two-hop
global route \citep{gilmer2017neural,ishiguro2019graph}. Shortening the path does
not by itself determine how much information the route can carry. If many
dependencies share one global state, the state and its readout can become a new
communication bottleneck. The relevant question is therefore not only whether
the shortcut is short, but whether its memory can be written and queried
efficiently.

This motivates using controlled oversquashing benchmarks to study the memory
interface of the shortcut. Two-Radius \citep{mishayev2025short} fixes the
source--target distance while increasing the number of associations that must
cross the cut, isolating finite-state capacity and retrieval. Tree-NeighborsMatch
\citep{alon2021bottleneck} increases both path length and the number of candidate
messages. Adding a global route removes the long-path component; the remaining
question is how effectively the auxiliary state preserves and retrieves the
required associations.

We ask what makes a global shortcut an \emph{addressable memory}, rather than
only a pooled global state. Different sources should be able to write distinct
contents, and each target should be able to select the content it needs. We
show that this addressable interface needs only three generic ingredients, so
cross-attention is one realization rather than the definition itself. We also
measure what an ordinary additive VN pays to emulate the same addressing
operation. The second property is \emph{support awareness}: a normalized read
keeps relative weights but can lose how strongly the memory matches a query.
We add a private reference inside the normalization to keep this information.

\paragraph{Contributions.}
\begin{itemize}[leftmargin=*,itemsep=1pt,topsep=2pt]
\item \textbf{A general account of addressable memory.} Margin codes in
$p=O(\log M)$ dimensions, a margin-amplifying nonlinearity, and multiplicative
write/read maps give $M$ selectable memory rows with explicit leakage bounds.
Exact one-hot responses from an affine address map require $p\ge M-1$
(Section~\ref{sec:addressable}).
\item \textbf{A scoped cost result for implicit addressing.} Within the dense
additive architecture formalized in Proposition~\ref{prop:additive-cost}, exact
recovery of arbitrary continuous payloads requires write and read bottlenecks
of width at least $Md$, hence $(Md)^2$ weights in the corresponding dense
layers. Multiplicative memories make the same routing product explicit with
shared projections whose parameter count is near-linear in $M$ for fixed node
and payload widths.
\item \textbf{Support-aware normalization.} A normalized read can hide the
absolute support behind its relative weights. A private reference keeps the
read bounded while exposing this support; cardinality recovery is a special
case (Section~\ref{sec:counting}).
\item \textbf{Controlled and real-graph experiments.} Two-Radius,
Tree-NeighborsMatch, and semantic routing separate addressability from raw
memory size and optimization effects. Standard LRGB tasks provide a transfer
check in settings where the global bottleneck is not expected to dominate
(Section~\ref{sec:benchmarks}).
\end{itemize}

Dense attention between all nodes would also give a global path, but at cost
$O(n^2d)$. We keep the MPNN as the local processor and add a bounded memory with
$O(nMd)$ cost for $M$ rows.

\section{Virtual Nodes and the Oversquashing Cut}
\label{sec:cut}

\paragraph{Two-Radius.}
Let $\calS=\{s_1,\ldots,s_n\}$ be sources and $\calT=\{t_1,\ldots,t_n\}$
targets. A few central nodes are adjacent to all of them, with no direct
source--target edge. Source $s_i$ carries $(i,\pi(i))$ for a permutation $\pi$
of $[n]$, and target $t_j$ must output $y_{t_j}=\pi(j)$. Every target is two hops
from every source, so a VN adds no shorter path. Yet all $n$ associations must
cross the cut between sources and targets, and the targets must tell $n!$
assignments apart. A VN is a second route across the same cut
(Figure~\ref{fig:overview}); the question is how much it can carry and how well
targets can read it.

We call the number of source-dependent states that can vary independently and
be selected separately by the decoder the \emph{effective factorization degree}
$L_{\mathrm{eff}}$. Identical VN clones count as one state
(Appendix~\ref{app:capacity}); separately selectable memory rows count
separately.

\begin{theorem}[Finite-capacity requirement]
\label{thm:lower}
Any deterministic architecture that solves every permutation-valued
Two-Radius instance exactly, with effective factorization degree
$L_{\mathrm{eff}}$, width $d$, and $b$-bit coordinates, satisfies
\begin{equation}
    L_{\mathrm{eff}}db \geq \log_2(n!).
    \label{eq:capacity}
\end{equation}
Hence $d=\Omega(n\log n/(L_{\mathrm{eff}}b))$.
\end{theorem}

The proof counts transcripts: different permutations need different
source-dependent states (Appendix~\ref{app:capacity}). So raw size matters. But
size is not enough: the targets must also be able to select their own entry.

\begin{proposition}[A separable broadcast cannot implement arbitrary lookup]
\label{prop:separable-vn}
If the source-dependent contribution to target $j$ has the form
$\widehat y_j=f(h_j)+g(z)$, where $h_j$ depends only on the target and $z$ is one
homogeneous global state, it cannot realize every Two-Radius instance.
\end{proposition}

Swapping the labels of two targets $p$ and $q$ changes $g(z)$ in the same way for
every target, while $p$ and $q$ need different changes. A joint decoder
$F(h_j,z)$ can escape this, but only through an interaction between the target
and the stored state. The rest of the paper makes this interaction explicit and
measures its cost.

\paragraph{Tree-NeighborsMatch.}
Tree-NeighborsMatch \citep{alon2021bottleneck} places candidate labels at the
leaves of a complete binary tree and gives each candidate a distinct local
neighbor-count feature. The root receives one such count and must predict the
label of the matching candidate. The count is supplied as an input feature in
the benchmark, so the difficulty is communication rather than learning the
count itself. Without a shortcut, information from $2^r$ candidates travels
through the tree to the root. A global route shortens this communication, while
still requiring the auxiliary state to preserve the candidate--label
associations.

\paragraph{What a good shortcut needs.}
Both benchmarks remove the distance argument and keep the capacity one. They
therefore test the two things a VN must do beyond being short: store many
entries without mixing them, and let each target retrieve its own entry. The
first is a matter of raw size (Theorem~\ref{thm:lower}); the second is a matter
of organization, which we call addressability. When sources can be repeated, a
third need appears: the read must keep how much evidence supports it, which we
call support awareness.

\begin{figure}[t]
\centering
\resizebox{0.84\linewidth}{!}{%
\begin{tikzpicture}[
  font=\small,
  real/.style={circle,draw=black!55,line width=.5pt,minimum size=4.8mm,inner sep=0pt},
  src/.style={real,fill=sourceblue!25},
  tgt/.style={real,fill=targetorange!30},
  central/.style={circle,draw=black!75,fill=black!10,minimum size=8.2mm,inner sep=0pt,line width=.7pt},
  vn/.style={circle,draw=slotgreen!65!black,fill=slotgreen!13,minimum size=8.5mm,inner sep=0pt,line width=.7pt,dashed},
  slot/.style={rectangle,rounded corners=2.5pt,draw=slotgreen!65!black,fill=slotgreen!16,minimum width=9.5mm,minimum height=5.4mm,line width=.6pt},
  graph/.style={-{Latex[length=1.7mm]},line width=.48pt,black!55},
  global/.style={-{Latex[length=1.7mm]},line width=.48pt,slotgreen!65!black,dashed},
  faint/.style={-{Latex[length=1.5mm]},line width=.4pt,black!20},
  hot/.style={-{Latex[length=2mm]},line width=1pt,accentred},
  coltitle/.style={font=\scriptsize\itshape,text=black!55},
  ptitle/.style={font=\bfseries\small}
]
\begin{scope}
  \node[ptitle] at (2.0,2.55) {(a) Additive VN};
  \node[coltitle] at (0,1.95) {sources};
  \node[coltitle] at (4.0,1.95) {targets};
  \foreach \y [count=\k] in {1.45,0.85,0.25,-0.95}{\node[src] (as\k) at (0,\y) {};}
  \node at (0,-0.35) {$\vdots$};
  \foreach \y [count=\k] in {1.45,0.85,0.25,-0.95}{\node[tgt] (at\k) at (4.0,\y) {};}
  \node at (4.0,-0.35) {$\vdots$};
  \node[vn] (avn) at (2.0,1.10) {\scriptsize VN};
  \node[central] (ac) at (2.0,-0.60) {$c$};
  \foreach \k in {1,2,3,4}{
    \draw[graph] (as\k) -- (ac);
    \draw[graph] (ac) -- (at\k);
    \draw[global] (as\k) -- (avn);
    \draw[global] (avn) -- (at\k);
  }
\end{scope}
\begin{scope}[xshift=6.1cm]
  \node[ptitle] at (2.0,2.55) {(b) Addressable memory};
  \node[coltitle] at (0,1.95) {sources};
  \node[coltitle] at (4.0,1.95) {targets};
  \foreach \y [count=\k] in {1.45,0.85,0.25,-0.95}{\node[src] (bs\k) at (0,\y) {};}
  \node at (0,-0.35) {$\vdots$};
  \foreach \y [count=\k] in {1.45,0.85,0.25,-0.95}{\node[tgt] (bt\k) at (4.0,\y) {};}
  \node at (4.0,-0.35) {$\vdots$};
  \node[slot] (z1) at (2.0,1.40) {$S_1$};
  \node[slot] (z2) at (2.0,0.68) {$S_2$};
  \node at (2.0,0.15) {$\vdots$};
  \node[slot] (zM) at (2.0,-0.37) {$S_M$};
  \node[central] (bc) at (2.0,-1.20) {$c$};
  \foreach \k in {1,2,3,4}{
    \draw[graph] (bs\k) -- (bc);
    \draw[graph] (bc) -- (bt\k);
    \foreach \m in {1,2,M}{
      \draw[faint] (bs\k) -- (z\m);
      \draw[faint] (z\m) -- (bt\k);
    }
  }
  \draw[hot] (bs2) -- (z2);
  \draw[hot] (z2) -- (bt2);
\end{scope}
\end{tikzpicture}}
\caption{Two routes across the Two-Radius cut. \textbf{(a)} An additive VN adds
one pooled state broadcast to every target. \textbf{(b)} An addressable
multiplicative memory adds $M$ separately selectable rows; the highlighted route
is a source-specific write followed by a target-specific read. In both cases
the graph path through the central node $c$ is unchanged.}
\label{fig:overview}
\end{figure}
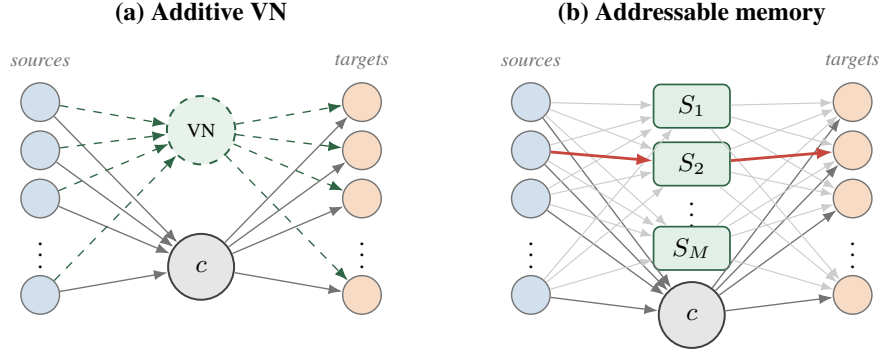

\section{Addressable Global Memory}
\label{sec:addressable}

\subsection{Additive and multiplicative memories}

Let $h_i$ be node states after a local MPNN block. Sources are mapped to routing
keys $k_i=\eta_{\mathrm w}(h_i)\in\R^p$ and payloads $v_i\in\R^d$; targets are
mapped to routing queries $q_j=\eta_{\mathrm r}(h_j)\in\R^p$.

\begin{definition}[Additive and multiplicative memories]
\label{def:memories}
An \emph{additive} memory pools $z=\sum_i\psi(h_i)$ (or the mean) with a generic
map $\psi$ and reads $F(h_j,z)$. An \emph{addressable multiplicative} memory
with $M$ rows uses nonnegative address maps
$\phi^{\mathrm w},\phi^{\mathrm r}:\R^p\to\R^M_{\ge0}$ that share a row index,
and writes
\begin{equation}
\widetilde S=\sum_i \phi^{\mathrm w}(k_i)\,v_i^\top\in\R^{M\times d},
\qquad
Z=\sum_i\phi^{\mathrm w}(k_i)\in\R^M_{\ge0}.
\label{eq:mult-write}
\end{equation}
It reads either row-normalized,
$o_j=\sum_m\rho_m(q_j)\,\widetilde S_m/Z_m$ with
$\rho(q)=\phi^{\mathrm r}(q)/\mathbf 1^\top\phi^{\mathrm r}(q)$, or jointly
normalized, $o_j=\phi^{\mathrm r}(q_j)^\top\widetilde S\,/\,\phi^{\mathrm r}(q_j)^\top Z$.
\end{definition}

In the idealized one-hot case, a multiplicative write is
\begin{equation}
\phi^{\mathrm w}(k_i)\,v_i^\top\;\simeq\;e_{g(i)}\otimes v_i,
\label{eq:implicit-gating}
\end{equation}
so the payload of source $i$ lands in row $g(i)$ by construction. A linear map
of $[k_i;v_i]$ cannot produce this product. An additive memory has to build it
with a nonlinear network and learn the matching read. We call this
\emph{implicit addressing}. It is possible with enough capacity, so our claims
are about cost, not about what can be expressed.

\paragraph{Two realizations.}
The definition describes an interface rather than a specific layer. We use two
realizations to make this explicit. \textbf{Cross-attention slots} use learned
slot identities $A\in\R^{M\times p}$, with write and read projections
$A_{\mathrm w}$ and $A_{\mathrm r}$. Their write and read are
\begin{align}
A^{\mathrm{in}}&=\softmax_{\mathrm{src}}
\!\left(\frac{A_{\mathrm w}R_{\calS}^{\top}}
{\tau_{\mathrm{in}}\sqrt p}\right),
& S&=S^{(0)}+A^{\mathrm{in}}H_{\calS}W_VW_O,\label{eq:slotwrite-main}\\
A^{\mathrm{out}}&=\softmax_{\mathrm{slots}}
\!\left(\frac{R_{\calT}A_{\mathrm r}^{\top}}
{\tau_{\mathrm{out}}\sqrt p}\right),
& O_{\calT}&=A^{\mathrm{out}}SU_V.\label{eq:slotread-main}
\end{align}
Thus the same indexed axis is used on both sides: it is an explicit
\emph{address axis} for organizing incoming payloads and selecting outgoing
content. This is closely related to induced-set and Perceiver-style latent
arrays \citep{lee2019set,jaegle2021perceiver}.

\textbf{Bilinear memories} use the linear-attention fast-weight state
\citep{katharopoulos2020transformers},
\begin{equation}
\widetilde S=\sum_i\phi^{\mathrm w}(k_i)v_i^\top,\qquad
o_j=\frac{\phi^{\mathrm r}(q_j)^\top\widetilde S}
{\phi^{\mathrm r}(q_j)^\top Z},
\quad \phi(x)=\ELU(Wx+b)+1.
\label{eq:bilinear-main}
\end{equation}
There is no explicit slot object, but the $M$ feature channels form the same
addressable row axis. In both realizations, compact address codes select rows
before the payload is decoded; the difference is whether the rows are explicit
latents or feature channels of a bilinear state.

\paragraph{Reading into the node state.}
Addressability determines which memory content is retrieved; the retrieved
content must still affect the target update. In both realizations we use the
same small FiLM integration,
\begin{equation}
[\gamma_j,\beta_j]=\MLP_{\mathrm{film}}([h_j,o_j]),\qquad
h_j^+=(1+\gamma_j)\odot h_j+\beta_j,
\label{eq:slotfilm}
\end{equation}
initialized near the identity. This keeps the comparison focused on the memory
interface: source order is removed by the write, target order is preserved by
the read, and the addressable memory only organizes distinctions already
present in the node states.

\subsection{Addressing from margin amplification}

As in the constructive Two-Radius setting, split the sources into groups
$I_m=\{i:g(i)=m\}$; each source carries the code of its group and each target the
code of its own group.

\begin{lemma}[Compact margin codes]
\label{lem:codes}
For every $\varepsilon\in(0,1)$ and $M\ge2$, there are unit vectors
$a_1,\ldots,a_M\in\R^p$ with $|\langle a_m,a_{m'}\rangle|<\varepsilon$ for all
$m\ne m'$ as soon as $p>4\ln M/\varepsilon^2$. So $M$ codes with constant margin
$\Delta=1-\varepsilon$ exist in $p=O(\log M)$ dimensions.
\end{lemma}

Random $\pm1/\sqrt p$ codes work: each inner product is a sum of $p$ bounded
independent terms, so Hoeffding's inequality and a union bound over the
$M^2$ pairs give the claim.

\begin{lemma}[Affine one-hot addressing needs many dimensions]
\label{lem:affine}
If $\phi(k)=Wk+b$ with $W\in\R^{M\times p}$, the response matrix
$\Phi_{gm}=\phi_m(a_g)$ has rank at most $p+1$. Exact one-hot gates, $\Phi=I_M$,
therefore require $p\ge M-1$.
\end{lemma}

Codes alone only give a constant gap: the matching score is $1$ and the others
are at most $1-\Delta$. An elementwise nonlinearity can turn this gap into a
large ratio. We measure it by the \emph{amplification ratio} of a nonnegative,
nondecreasing $\sigma$ with $\sigma(1)>0$,
\begin{equation}
\kappa_\sigma(\Delta)=\sup_{x\le1-\Delta}\sigma(x)\,/\,\sigma(1).
\label{eq:amplification}
\end{equation}

\begin{theorem}[Addressing from margin amplification]
\label{thm:addressing}
Let $a_1,\ldots,a_M$ be unit codes with margin $\Delta$, and
$\phi^{\mathrm w}_m(k)=\phi^{\mathrm r}_m(k)=\sigma(\langle a_m,k\rangle)$ with
$\kappa=\kappa_\sigma(\Delta)$. (i) With row normalization, the mass that row
$m$ receives from outside its group and the read weight of a target of group $g$
on its own row satisfy
\begin{equation}
\delta_m\le\frac{n-|I_m|}{|I_m|}\,\kappa,
\qquad
\rho_{gg}\ge\frac{1}{1+(M-1)\kappa}.
\label{eq:routing-margin}
\end{equation}
(ii) With joint normalization, the weight that a target of group $g$ puts on
sources outside $I_g$ satisfies
\begin{equation}
\varepsilon_g\le\frac{n-|I_g|}{|I_g|}\big(2\kappa+(M-2)\kappa^2\big).
\label{eq:joint-margin}
\end{equation}
\end{theorem}

\paragraph{Proof sketch.}
Rescale $\sigma$ so that $\sigma(1)=1$; both normalizations are invariant to this.
With row normalization, row $m$ gives weight $1$ to each of the $|I_m|$ sources
of its group and at most $\kappa$ to each outsider, so the outside share is at
most $(n-|I_m|)\kappa/|I_m|$; a target of group $g$ gives weight $1$ to its row
and at most $\kappa$ to each of the others. With joint normalization, the weight
of source $i$ is the kernel $\sum_m\sigma(\langle a_m,a_g\rangle)
\sigma(\langle a_m,a_{g(i)}\rangle)$: it is at least $1$ inside the group and, for
an outsider, at most $\kappa$ for $m=g$, $\kappa$ for $m=g(i)$ and $\kappa^2$ for
the others. Full proofs are in Appendix~\ref{app:addressing}.

\begin{corollary}[Softmax, $\ELU+1$ and threshold memories]
\label{cor:instances}
With the codes of Lemma~\ref{lem:codes}, $p=O(\log M)$ address dimensions and
$M$ rows suffice for (a) cross-attention slots, where $\sigma(x)=e^{x/\tau}$ gives
$\kappa=e^{-\Delta/\tau}$; (b) $\ELU+1$ bilinear memories whose $m$th
pre-activation is $\beta(\langle a_m,k\rangle-1)\le0$, which stays on the
exponential branch and gives $\kappa=e^{-\beta\Delta}$; and (c) thresholded maps
such as $\sigma(x)=\max(0,x-1+\Delta/2)$, for which $\kappa=0$. Leakage at most
$\eta$ needs $1/\tau$ or $\beta$ of order $\Delta^{-1}\log(n/(\eta\min_m|I_m|))$.
\end{corollary}

Softmax competition is therefore one way to build an addressable memory, not
the only one. What matters is a small off-target/on-target ratio together with
addressable multiplicative writes and reads. The $O(\log M)$ term refers only to the address
code: both realizations still keep an $M\times d$ dynamic payload state. The
condition on $\ELU+1$ is also substantive. The construction in
Corollary~\ref{cor:instances} keeps the relevant pre-activations on its negative,
exponential branch; a generic use of linear attention does not automatically
satisfy this margin-amplification condition. Thus the result identifies a
sufficient memory interface rather than claiming that every attention or
linear-attention layer is compactly addressable. Full proofs and additional
rank statements are in Appendix~\ref{app:addressing}.

\subsection{Capacity, and what an additive memory pays}

Rows that can be selected separately count separately, so
Theorem~\ref{thm:lower} applies with $L_{\mathrm{eff}}=M$, and a matching
construction exists.

\begin{theorem}[Addressable construction]
\label{thm:upper}
For $M\le n$, a permutation-invariant multiplicative write with $M$ rows and a
permutation-equivariant decoder solve Two-Radius with
\begin{equation}
 d\le
 \left\lceil\tfrac{n}{M}\right\rceil
 \left\lceil\tfrac{\log_2 n}{b}\right\rceil
 +\left\lceil\tfrac{\log_2 M}{b}\right\rceil+O(1)
 \label{eq:upper}
\end{equation}
coordinates per row, for both row- and jointly normalized memories.
\end{theorem}

Source $i$ writes the code of $\pi(i)$ into block $r(i)$ of row $g(i)$, and target
$i$ reads the same row and block. This is a constructive bound; it does not say
that training finds this partition.

\begin{proposition}[Linear read rank]
\label{prop:rank}
For a memory $S\in\R^{M\times d}$ read linearly by $n$ targets as $O=RS$,
$\operatorname{rank}(O)\le M$. Exact recovery of $n$ arbitrary payloads of
width $d\ge n$ therefore requires $M\ge n$.
\end{proposition}

When $n>M$, some rows must store several payloads, and selecting one of them
requires additional target-dependent decoding. The row axis reduces the routing
problem; it does not remove the finite rank of a linear read.

An additive VN can run the same construction: ReLU networks compute margin
gates and multiply them with payloads. The difference is the price. We make it
precise in the idealized setting above, with one source per group ($n=M$); node
states carry their code and, for sources, their payload.

\begin{proposition}[Additive emulation and its cost]
\label{prop:additive-cost}
Let an additive memory write $z=b+\sum_gL\,\xi(h_g)$, that is $\psi=L\,\xi$ in
Definition~\ref{def:memories}, with $\xi$ continuous and $L\in\R^{D\times W}$ its
last linear layer, and read $F(h_j,z)=G(Az+Bh_j+c)$ with first decoder layer
$A\in\R^{W'\times D}$.
(i) \emph{Emulation.} If $\|v_g\|_\infty\le C$, ReLU networks $\xi$ and $G$ with
two hidden layers and widths $O(M+Md)$ recover every payload exactly, with
address codes of dimension $p=O(\log M)$.
(ii) \emph{Cost.} If the memory recovers every payload exactly for all payloads
in $U^M$, with $U\subset\R^d$ open, then $D\ge Md$, $W\ge Md$ and $W'\ge Md$: the
last write layer and the first read layer each have at least $(Md)^2$ weights.
A multiplicative realization with shared projections uses
$O(Mp+hp+p^2+hd+d^2)$ parameters, where $h$ is the node width.
\end{proposition}

The lower bound is specific to the dense additive factorization stated in the
proposition. On the write side, the target states $e_g\otimes v$ span all $Md$
directions, so the final dense map must have rank at least $Md$. On the read
side, the first dense map acts on the full state before the target-dependent
nonlinearity can select a block; retaining arbitrary continuous payloads again
requires an $Md$-dimensional bottleneck. A multiplicative memory performs the
address--payload product before pooling and selects along the memory axis before
the small payload decoder. The result therefore formalizes one cost of implicit
addressing for ordinary dense MLP interfaces; it is not a lower bound on every
possible structured VN implementation.

\paragraph{Scope and trade-offs.}
Addressability does not remove finite-memory limits. A decoder that is linear in
the memory contents can recover at most $M$ arbitrary payload directions, so
packing more than $M$ entries into $M$ rows requires additional conditional
structure in the target decoder (Proposition~\ref{prop:rank}). Compact addresses
also rely on a non-vanishing margin. At fixed $p$, the number of well-separated
codes is finite; increasing $M$ must eventually reduce the available margin or
require a larger $p$. Sharper amplification can compensate in the bounds, but
also changes the optimization landscape. At a fixed state budget $Md$, more
rows additionally imply narrower payloads. These trade-offs are the reason the
logarithmic address dimension should not be read as logarithmic memory capacity.

The same distinction is visible in a degenerate routing regime. If the write
scores are independent of the source, every row receives the same pooled
source-dependent component and the effective factorization falls back to one,
even if the rows have different learned initial states. In practice this can be
monitored through write entropy and read usage. The guarantees above are
representational: row-normalized slots and jointly normalized bilinear memories
can satisfy the same addressing bounds while exhibiting different training
dynamics. Appendix~\ref{app:addressing} gives the corresponding details.

\section{Preserving Query Support}
\label{sec:counting}

Addressability decides \emph{where} information is stored. A second question is
how strongly the sources support a read. Let $w(q,x)\ge0$ be the weight that
query or row $q$ gives to source $x$, and
\begin{equation}
Z(q,X)=\sum_{x\in X}w(q,x)
\label{eq:supportmass}
\end{equation}
the \emph{support} of the read. It is not just a count: each query learns its own
notion of a match.

\begin{proposition}[Normalized reads are blind to replication]
\label{prop:replication}
For a row-normalized or jointly normalized read, replicating every source $r$
times leaves the read unchanged, while $Z(q,rX)=rZ(q,X)$. An unnormalized read
scales with $r$.
\end{proposition}

Numerator and denominator grow by the same factor. Thus, when the public
interface exposes only the normalized read, uniform replication is not
observable through that read. This includes row-normalized slots and normalized
linear-attention reads unless they expose an additional denominator or mass
state. An unnormalized sum retains the scale but loses the bounded convex-read
property. Self-attention often has an internal reference because the query token
also contributes its own key; a latent source-to-memory query need not have an
analogous element.

\paragraph{A private reference.}
For each row (and head), we add one reference with weight $w_0=e^{s_0}$ and value
$a$ that is not a source:
\begin{equation}
z(X)=\frac{w_0\,a+\sum_i w_iv_i}{w_0+\sum_i w_i},
\qquad
Z=w_0\,\frac{1-\alpha_0}{\alpha_0},\quad \alpha_0=\frac{w_0}{w_0+Z}.
\label{eq:anchored}
\end{equation}
The normalized content $\bar v=\sum_i(w_i/Z)v_i$ and the support relative to the
reference are then both available. We learn $s_0$ per row/head and feed the
log-odds $\log Z-s_0$ to a small MLP that bypasses content normalization. For a
jointly normalized memory, the same construction adds a positive reference
vector $c_0$ to the denominator and $c_0\odot a$ to the numerator. If an
implementation already exposes its raw denominator accumulator, that quantity
itself provides a mass path and the reference is not needed for observability.

\begin{proposition}[Anchored reads are bounded and support-aware]
\label{prop:anchor-bounded}
The read of Eq.~\eqref{eq:anchored} is a convex combination of $a$ and the source
values, so $\|z(X)\|\le\max(\|a\|,\max_i\|v_i\|)$ for any number of sources. Its
log-odds $\log Z-s_0$ are strictly increasing in $Z$.
\end{proposition}

With a zero anchor value,
\[
z(X)=\frac{Z}{e^{s_0}+Z}\,\bar v,
\]
so $s_0$ sets the support scale at which the content becomes active.
Softmax$_1$ \citep{miller2023attentionoffbyone} is the fixed case $s_0=0$.
In self-attention it lets a head put little mass on real tokens; here the
reference instead supplies the ``self'' term that a latent cross-attention
query does not naturally have, and its weight provides the support signal.

\begin{theorem}[Cardinality as a special case]
\label{thm:wlinjective}
On multisets of bounded size over a finite color alphabet, anchored attention
followed by an expressive map represents $X$ injectively, and can therefore
implement one 1-WL refinement on bounded colored neighborhoods. A normalized read
without a separate mass path is not injective on any domain containing $X$ and
$rX$, $r>1$.
\end{theorem}

With zero source logits, one-hot color values and a zero-logit reference, the
reference weight is $1/(1+|X|)$ and the other coordinates hold the color
histogram scaled by $|X|/(1+|X|)$, which recovers both (Appendix~\ref{app:mass}).

\section{Experiments}
\label{sec:benchmarks}

We test whether making the global shortcut explicitly addressable helps when
oversquashing is controlled: Two-Radius fixes distance while the number of
associations grows, whereas Tree-NeighborsMatch increases both. On these tasks
we evaluate two addressable realizations, cross-attention slots and the anchored
bilinear memory. We then use permutation-invariant semantic routing to test
content-based addressing without persistent identifiers, and close with a short
summary of standard LRGB benchmarks. Protocols are in the appendix.

\subsection{Two-Radius}
\label{sec:two-radius}

\paragraph{Addressability.}
On identifier-based Two-Radius with a $192$-scalar memory, we remove the two
halves of the address interface in turn (Table~\ref{tab:write-read-main}). Full
slots, with routed writes and target-to-slot reads, are exact on all seeds within
$500$ updates. Keeping routed writes but reading the flattened memory with a
target-conditioned MLP eventually reaches $94.3\%$, with $8\times$ more
parameters, and only two of three seeds recover full graphs. Without routed
writes, the rows receive the same summary and accuracy stays below $60\%$. A
strong implicit decoder can thus read a routed memory, as
Proposition~\ref{prop:additive-cost} allows, but in this experiment it requires
a larger decoder and shows greater variability across seeds
(Appendix~\ref{app:write-read-ablation}).

\begin{table}[h]
\centering
\small
\setlength{\tabcolsep}{4pt}
\caption{Two-Radius write/read ablation: target accuracy (\%) after a given
number of updates, mean $\pm$ s.d. over three seeds.}
\label{tab:write-read-main}
\begin{tabular}{@{}lrrrr@{}}
\toprule
Model & Params & 500 & 2k & 5k\\
\midrule
No routed write & 3.85k & $10.7\pm0.5$ & $49.0\pm12.3$ & $58.6\pm13.5$\\
Routed write, implicit read & 32.2k & $79.2\pm2.7$ & $93.9\pm10.6$ & $94.3\pm9.8$\\
Full routed write/read & 3.85k & $\mathbf{100.0}$ & $\mathbf{100.0}$ & $\mathbf{100.0}$\\
\bottomrule
\end{tabular}
\end{table}

\paragraph{Support.}
In multiplicity-aware Two-Radius, each identifier $i$ has a label $\pi(i)$ and a
base multiplicity $a_i\in\{1,\ldots,4\}$; each assignment is rendered at scales
$r\in\{1,2,3\}$, so source $(i,\pi(i))$ appears $c_i=ra_i$ times, and target $t_i$
predicts $(\pi(i),c_i)$. The label tests addressing; the count tests support.

\begin{table}[t]
\centering
\small
\setlength{\tabcolsep}{4pt}
\caption{Multiplicity-aware Two-Radius, validation accuracy (\%), single seed.}
\label{tab:counted}
\begin{tabular}{@{}lrrr@{}}
\toprule
Model & Label & Count & Both\\
\midrule
MPNN & 7.8 & 17.1 & 1.3\\
Additive VN & 8.9 & 18.3 & 1.4\\
Normalized slots & 100.0 & 28.3 & 28.3\\
Anchored slots & \textbf{100.0} & \textbf{100.0} & \textbf{100.0}\\
Anchored bilinear memory & \textbf{100.0} & \textbf{100.0} & \textbf{100.0}\\
\bottomrule
\end{tabular}
\end{table}

Table~\ref{tab:counted} separates the two mechanisms in a controlled run.
Normalized slots recover the label associations but not the common replication
scale, consistent with Proposition~\ref{prop:replication}. Both anchored
realizations recover the label and count outputs in this seed, showing that the
reference mechanism is not specific to the row-normalized slot implementation.
We use this experiment as a mechanistic test of the information exposed by the
normalization interface, not as a general performance comparison.

\subsection{Tree-NeighborsMatch}
\label{sec:tree}

We follow Tree-NeighborsMatch \citep{alon2021bottleneck}: candidate leaves carry
labels and distinct local neighbor-count features, while the root receives one
such count and must predict the label of the matching candidate. As $r$ grows,
$2^r$ candidate associations must reach the root. The benchmark supplies the
count feature, so the comparison focuses on communication. We evaluate a local
MPNN, pooled homogeneous VNs, and addressable slot memories under the same data
generator. We also report a separate diagnostic for the standard persistent
OGB-style virtual node in Appendix~\ref{app:tree}.

\begin{table}[t]
\centering
\small
\setlength{\tabcolsep}{5pt}
\caption{Tree-NeighborsMatch test accuracy (\%) by tree depth $r$ ($2^r$
candidates/classes; chance is $2^{-r}$). Mean $\pm$ s.d. over three seeds.
Parameters are reported at $r=5$.}
\label{tab:tree-neighbors}
\begin{tabular}{@{}lrrrrr@{}}
\toprule
Model & Params & $r=3$ & $r=4$ & $r=5$ & $r=6$\\
\midrule
MPNN ($r+1$ layers) & 92k & $99.6\pm0.7$ & $16.9\pm0.9$ & $3.1\pm0.2$ & $5.1\pm0.2$\\
Pooled VN & 363k & $100.0$ & $73.9\pm0.2$ & $10.6\pm0.1$ & $5.3\pm0.2$\\
Wide pooled VN & 2.33M & $100.0$ & $75.2\pm5.5$ & $10.7\pm0.3$ & $5.7\pm0.2$\\

Slots & 37.5k & $\mathbf{100.0}$ & $\mathbf{100.0}$ & $\mathbf{100.0}$ & $\mathbf{26.2\pm0.2}$\\
Anchored slots & 38.9k & $\mathbf{100.0}$ & $\mathbf{100.0}$ & $99.9\pm0.1$ & $25.7\pm0.2$\\
\bottomrule
\end{tabular}
\end{table}

Table~\ref{tab:tree-neighbors} shows the clearest separation at $r=5$.
The local MPNN reaches $3.1\pm0.2\%$, and the two pooled VN variants reach about
$10.6\%$, whereas the slot variants remain near $100\%$ with substantially
fewer parameters. The same pattern continues at $r=6$: slots reach
$26.2\pm0.2\%$ and anchored slots $25.7\pm0.2\%$, while the pooled baselines
remain between $5.1\%$ and $5.7\%$. An anchored bilinear memory gives a second
addressable realization, reaching $100\%$ at $r=3,4,5$ and a mean of
$20.9\%$ at $r=6$. The reference has little effect on the slot results here,
as expected for a task that tests addressability rather than support scale. A
separate persistent OGB-style VN check is reported in
Appendix~\ref{app:tree}.

\subsection{Permutation-invariant semantic routing}
\label{sec:semantic-routing}

Each graph is a fresh episode with eight sources and eight targets. Sources carry
a continuous key and a categorical payload; each target receives a perturbed
copy of one key and must predict its payload. Keys are resampled for every
graph, and there are no identifiers or orderings, so the address has to be
inferred from content in each episode. All models have exactly $192$ dynamic
memory scalars and are trained for $50$k updates.

The \emph{parameter-matched VN} keeps these $192$ scalars in one homogeneous
state and uses a small width-$5$, depth-$2$ nonlinear write/read path
($2.29$k parameters). The \emph{high-capacity VN} keeps the same
$192$-dimensional state but gives implicit addressing much more capacity:
three-layer nonlinear source writing and three-layer target-conditioned
decoding with hidden width $96$ ($58.95$k parameters). The addressable slot
model instead uses $32$ rows of width $6$ and a shared semantic projection for
source addresses and target queries ($2.09$k parameters). Each family is tuned
separately; further optimization details are in
Appendix~\ref{app:semantic-routing}.

Because the keys are resampled in every episode, the VN cannot memorize a fixed
semantic-type-to-coordinate convention. Any implicit addressable organization
must be created from the current source set and decoded again from the target
query. The slot model keeps a stable addressable row axis while allowing the
contents assigned to those rows to change from episode to episode.

\begin{table}[t]
\centering
\small
\setlength{\tabcolsep}{4pt}
\caption{Semantic routing with $192$ memory scalars. Mean $\pm$ s.d. over three
seeds at $50$k updates. OOD uses stronger key noise.}
\label{tab:hero-routing}
\begin{tabular}{@{}lrrrrr@{}}
\toprule
Model & Params & MFLOPs & IID target & OOD target & OOD exact\\
\midrule
Parameter-matched VN & 2.29k & 0.033 & $30.7\pm0.3$ & $30.7\pm0.3$ & $0.0$\\
High-capacity VN & 58.95k & 0.934 & $93.9\pm1.0$ & $74.5\pm2.6$ & $9.6\pm2.6$\\
Cross-attention slots & 2.09k & 0.042 & $\mathbf{96.4\pm0.1}$ & $\mathbf{84.8\pm0.3}$ & $\mathbf{25.4\pm1.5}$\\
\bottomrule
\end{tabular}
\end{table}

Table~\ref{tab:hero-routing} compares the three models at the same
$192$-scalar dynamic-memory budget. The parameter-matched VN remains near its
low-accuracy solution under this training budget. The high-capacity VN shows
that a homogeneous state can learn strong implicit addressing when its writer
and target-conditioned decoder are made much larger. It reaches $93.9\%$ IID
accuracy, but uses $28\times$ more parameters and about $22\times$ more dense
compute than the explicit addressable slot model, and the gap is larger under
stronger key noise.
Proposition~\ref{prop:additive-cost} gives a scale comparison for continuous
payloads, but it is not a lower bound for this categorical benchmark. Here the
comparison asks a narrower empirical question: how much capacity is needed when
the addressable memory axis is learned implicitly rather than exposed by the
architecture.

\subsection{Standard benchmarks}
\label{sec:lrgb}

We finally use three LRGB tasks \citep{dwivedi2022long} as a transfer check:
Peptides-func, Peptides-struct, and PascalVOC-SP. A light, non-persistent
anchored slot block on top of an OGB-style VN gives similar five-seed mean
scores on Peptides-func and PascalVOC-SP and a lower mean MAE on
Peptides-struct. The anchored bilinear realization gives worse mean scores on these tasks
($0.680$ AP, $0.247$ MAE, and $0.389$ F1), despite matching the slot memory on
the controlled addressability benchmarks. The two addressable
parameterizations therefore behave similarly on the mechanism tests but not on
every downstream task. Full means and protocol details are in
Appendix~\ref{app:lrgb}.

\section{Related Work}
\label{sec:related}

\paragraph{Oversquashing and virtual nodes.}
Oversquashing has been studied through receptive-field growth, curvature and
communication bottlenecks
\citep{alon2021bottleneck,topping2022understanding,di2023over,loukas2020depth};
Two-Radius isolates its capacity part \citep{mishayev2025short}. Virtual nodes
give a cheap global shortcut \citep{gilmer2017neural,ishiguro2019graph}; other
work studies their mixing \citep{southern2025virtual,hwang2022analysis}, their
relation to attention and expressivity
\citep{cai2023connection,rosenbluth2024uniform}, and relays through attention
nodes \citep{caruso2026range}. We study how the memory of the shortcut should be
organized once it exists.

\paragraph{Latent arrays, linear attention and attention support.}
Our row-normalized realization is close to induced set attention and Perceiver
latents \citep{lee2019set,jaegle2021perceiver}; the jointly normalized one is the
fast-weight state of linear attention \citep{katharopoulos2020transformers},
whose feature dimension bounds what a linear decoder can recover
\citep{schlag2021linear}. Softmax$_1$ adds a fixed unit reference to the softmax
\citep{miller2023attentionoffbyone}, and gated attention lets heads abstain
\citep{bondarenko2023quantizable}; registers and attention sinks also add special
attention elements \citep{darcet2024registers,xiao2024streaming}. Injective
multiset aggregation underlies 1-WL expressivity
\citep{xu2019powerful,morris2019weisfeiler}; PNA, cardinality-preserving
attention and ACAM keep counts explicitly
\citep{corso2020pna,zhang2019cardinality,marcoccia2026netdiff}. Here cardinality
is one special case of observable support.

\section{Discussion and Conclusion}
\label{sec:discussion}

A virtual node removes long communication paths, but its shared state can still
become a bottleneck. We identify a simple route to an \emph{addressable} global
memory: separated address codes, margin amplification, and multiplicative
write/read maps create an explicit address axis, with $O(\log M)$ address-code
dimensions under a constant-margin assumption. Cross-attention slots and the
studied $\ELU+1$ bilinear memory are two realizations.

Normalized cross-attention raises a separate issue because a latent query has
no natural self-key reference. A learned private anchor supplies that reference
and exposes query support while keeping the read bounded. The controlled
Two-Radius and Tree-NeighborsMatch experiments support explicit addressability
and support-aware reads, while LRGB shows that different addressable
realizations need not behave identically on general downstream tasks, especially regarding optimization. This paves the way for further research to design powerful, addressable and support-aware virtual nodes.

\section*{Reproducibility Statement}
Complete proofs are in Appendices~\ref{app:addressing}--\ref{app:mass}.
Appendices~\ref{app:semantic-routing}--\ref{app:tree} specify the synthetic
generators, models, budgets and diagnostics, and Appendix~\ref{app:lrgb} the
real-graph protocol and all runs. Executable implementations and data generators
will be submitted as anonymized supplementary material.

\section*{AI Use Statement}
In this work, we used generative AI tools through prompt-accessible web interfaces to help draft and edit parts of the
manuscript text, to propose and check parts of the mathematical derivations and
proofs, and to write and debug experiment code. All experiments were run by the
author, and generative AI tools were not used to generate, alter or select data
or results. We have reviewed all AI-assisted work: the derivations were checked
by the author and the experiment code was tested before use. We take
responsibility for the final content of this work, including text, claims or
artifacts produced with the aid of generative AI. No autonomous agentic AI tool was used for this work.

\bibliography{references}
\bibliographystyle{unsrtnat}

\appendix

\section{The Two Realizations in Detail}
\label{app:realizations}

\paragraph{Cross-attention slots.}
Let $A=[a_1,\ldots,a_M]^\top\in\R^{M\times p}$ be learned slot identities, with
write and read projections $A_{\mathrm w}=q_{\mathrm w}(A)$ and
$A_{\mathrm r}=k_{\mathrm r}(A)$, and let $R_{\calS}$ and $R_{\calT}$ stack the
routing keys and queries. The write and the read are
\begin{align}
A^{\mathrm{in}}&=\softmax_{\mathrm{src}}
\!\left(\frac{A_{\mathrm w}R_{\calS}^\top}{\tau_{\mathrm{in}}\sqrt p}\right),
&
S&=S^{(0)}+\left(A^{\mathrm{in}}H_{\calS}W_V\right)W_O,
\label{eq:slotwrite}\\
A^{\mathrm{out}}&=\softmax_{\mathrm{slots}}
\!\left(\frac{R_{\calT}A_{\mathrm r}^\top}{\tau_{\mathrm{out}}\sqrt p}\right),
&
O_{\calT}&=A^{\mathrm{out}}SU_V.
\label{eq:slotreadout}
\end{align}
The slot identity is the address and $S_m$ is the payload stored at that
address. The same identities generate write queries and read keys, so slot $m$
is the same compartment on both sides.

\paragraph{Bilinear memory.}
The state is $\widetilde S=\sum_i\phi^{\mathrm w}(k_i)v_i^\top\in\R^{M\times d}$ with
$\phi(x)=\ELU(Wx+b)+1$ and $M$ feature channels, and a target reads
$\phi^{\mathrm r}(q_j)^\top\widetilde S\,/\,\phi^{\mathrm r}(q_j)^\top Z$. The input
address has dimension $p$ while the feature map has dimension $M$: the
$O(\log M)$ statement concerns the former. Standard linear attention applies
$\phi$ coordinatewise with $M=p$; its state then has only $p$ rows, which caps
the number of exactly recoverable payloads at $p$
(Proposition~\ref{prop:rank}), and on the affine branch of $\ELU+1$ overlaps
between codes stay of constant order. The guarantee of
Corollary~\ref{cor:instances}(b) instead lifts $p$-dimensional keys to $M$
features whose pre-activations stay on the exponential branch.

\section{Addressing Proofs}
\label{app:addressing}

\subsection{Proof of Lemma~\ref{lem:codes}}

Draw each coordinate of each $a_m$ independently and uniformly from
$\{\pm1/\sqrt p\}$, so every $a_m$ is a unit vector. For $m\neq m'$,
$\langle a_m,a_{m'}\rangle$ is a sum of $p$ independent variables in
$[-1/p,1/p]$ with zero mean, and Hoeffding's inequality gives
\[
\Pr\big(|\langle a_m,a_{m'}\rangle|\geq\varepsilon\big)\leq
2\exp\!\big(-p\varepsilon^2/2\big).
\]
A union bound over the $M(M-1)/2$ pairs bounds the probability that some pair
fails by $M^2\exp(-p\varepsilon^2/2)$, which is smaller than one when
$p>4\ln M/\varepsilon^2$. Hence a valid codebook exists. \qed

\subsection{Proof of Lemma~\ref{lem:affine}}

Stack the codes as rows of $C\in\R^{M\times p}$. The response matrix is
$\Phi=CW^\top+\mathbf 1b^\top$, a sum of a matrix of rank at most $p$ and a
matrix of rank at most one. Hence $\operatorname{rank}\Phi\leq p+1$, and
$\Phi=I_M$ requires $p+1\geq M$. \qed

\subsection{Proof of Theorem~\ref{thm:addressing}}

Both normalizations are invariant to a positive rescaling of $\sigma$, so assume
$\sigma(1)=1$; then $\sigma(\langle a_m,a_{m'}\rangle)\leq\kappa$ for $m\neq m'$.

\emph{(i) Row normalization.} Row $m$ gives weight
$\sigma(\langle a_m,a_{g(i)}\rangle)$ to source $i$: each of the $|I_m|$ group
members contributes $\sigma(1)=1$ and each outsider at most $\kappa$. The
outside mass is therefore
\[
\delta_m=\frac{\sum_{i\notin I_m}\sigma(\langle a_m,a_{g(i)}\rangle)}
{\sum_{i}\sigma(\langle a_m,a_{g(i)}\rangle)}
\leq\frac{(n-|I_m|)\kappa}{|I_m|}.
\]
A target of group $g$ gives row $g$ the read weight
$\rho_{gg}=\sigma(1)/\sum_m\sigma(\langle a_m,a_g\rangle)\geq1/(1+(M-1)\kappa)$.

\emph{(ii) Joint normalization.} The read of a target with query $a_g$ weights
source $i$ proportionally to $K_i=\sum_m\sigma(\langle a_m,a_g\rangle)
\sigma(\langle a_m,a_{g(i)}\rangle)$. For $i\in I_g$, the term $m=g$ equals
$\sigma(1)^2=1$, so $K_i\geq1$. For $i\in I_h$ with $h\neq g$, the term $m=g$ is
$\sigma(1)\sigma(\langle a_g,a_h\rangle)\leq\kappa$, the term $m=h$ is
$\sigma(\langle a_h,a_g\rangle)\sigma(1)\leq\kappa$, and each of the $M-2$
remaining terms is at most $\kappa^2$. Hence
$K_i\leq2\kappa+(M-2)\kappa^2$, and
\[
\varepsilon_g=\frac{\sum_{i\notin I_g}K_i}{\sum_iK_i}
\leq\frac{(n-|I_g|)\big(2\kappa+(M-2)\kappa^2\big)}{|I_g|}. \qquad\qed
\]

\subsection{Proof of Corollary~\ref{cor:instances}}

For $\sigma(x)=e^{x/\tau}$, $\sup_{x\leq1-\Delta}\sigma(x)/\sigma(1)=e^{-\Delta/\tau}$.
For the bilinear memory, choose the rows of both $W_k$ and $W_q$ as $\beta a_m$ and set $b_k=b_q=-\beta\mathbf 1$.
For unit address codes, the corresponding pre-activation is
nonpositive, where $\ELU(x)+1=e^x$; hence $\sigma(x)=e^{\beta(x-1)}$ and
$\kappa=e^{-\beta\Delta}$. For the thresholded map, $\sigma(x)=0$ for
$x\leq1-\Delta$ and $\sigma(1)=\Delta/2>0$, so $\kappa=0$. Combining with
Lemma~\ref{lem:codes} gives $p=O(\log M)$. Finally, the bound of
Eq.~\eqref{eq:routing-margin} is at most $\eta$ as soon as
$\kappa\leq\eta\min_m|I_m|/n$, which for the exponential maps holds when $1/\tau$
or $\beta$ is at least $\Delta^{-1}\log(n/(\eta\min_m|I_m|))$; the joint bound
behaves in the same way up to a constant. \qed

\subsection{Scope of the addressing results}

\paragraph{What $O(\log M)$ means.}
The logarithm concerns the dimension $p$ of the address code. It does not say
that $M$ compartments fit in $O(\log M)$ dynamic scalars: cross-attention uses
$M$ slots and the bilinear memory $M$ feature channels, and both keep an
$M\times d$ payload state, plus an $M$-dimensional normalizer for joint
normalization. The gain is that a node does not need an $M$-dimensional one-hot
identifier to select a row.

\paragraph{No entirely free lunch.}
The result assumes that the model learns $M$ addresses with a constant margin
$\Delta$. For a fixed $p$, only a limited number of unit codes keep a given
separation, so increasing $M$ eventually shrinks the achievable margin. Keeping
the bounds of Theorem~\ref{thm:addressing} small then needs a larger $p$ or a
stronger amplification (lower $\tau$, larger $\beta$), which makes routing sharper
and potentially harder to optimize. Addresses also trade against payload: at a
fixed budget $Md$, more rows means narrower rows, and each read returns a single
mixture of rows. Theorem~\ref{thm:upper} makes this explicit, since each row must
hold $\lceil n/M\rceil$ label blocks of $\lceil\log_2 n/b\rceil$ coordinates.
At the other extreme, if the write scores are constant over sources for every row
(for instance $q_{\mathrm w}\equiv0$), every row receives the same pooled summary,
$L_{\mathrm{eff}}=1$, and the memory behaves like one mean-pooled VN
(Lemma~\ref{lem:vn-symmetry}). This regime is observable: the write entropy stays
at its uniform value and few rows receive read mass.

\begin{remark}[Optimization scope]
\label{rem:optimization}
The addressing bounds are representational and do not imply identical training
dynamics across row-normalized and jointly normalized realizations. The main
text discusses this distinction; the bounds themselves concern only the
existence and leakage of the address interface.
\end{remark}

\subsection{Proof of Proposition~\ref{prop:rank}}

For fixed targets, a decoder linear in the memory contents produces
$O=R\,S$ with $S\in\R^{M\times d}$ the memory and $R\in\R^{n\times M}$ the read
coefficients, so $\operatorname{rank}O\leq M$. Arbitrary payloads of width
$d\geq n$ can form a matrix of rank $n$, which requires $M\geq n$. \qed

\subsection{Proof of Proposition~\ref{prop:additive-cost}}

\emph{(i) Emulation.} Take the codes of Lemma~\ref{lem:codes} with margin
$\Delta$. The gate of row $m$,
\[
\gamma_m(k)=\mathrm{ReLU}(s_m)-\mathrm{ReLU}(s_m-1),\qquad
s_m=\tfrac{2}{\Delta}\big(\langle a_m,k\rangle-1+\tfrac{\Delta}{2}\big),
\]
equals $1$ when $k=a_m$ (then $s_m=1$) and $0$ when
$\langle a_m,k\rangle\leq1-\Delta$ (then $s_m\leq-1$). It uses two ReLU units per
row. For a scalar $x$ with $|x|\leq C$ and a gate $\gamma\in\{0,1\}$,
\[
\gamma x=\mathrm{ReLU}\big(x-C(1-\gamma)\big)-\mathrm{ReLU}\big(-x-C(1-\gamma)\big),
\]
since the right-hand side is $x$ when $\gamma=1$ and $0$ when $\gamma=0$.

\emph{Write.} The first hidden layer of $\xi$ computes the $2M$ gate units from
the code and passes the payload through $2d$ units
($x=\mathrm{ReLU}(x)-\mathrm{ReLU}(-x)$). The second hidden layer applies the
product identity to every pair (row $m$, coordinate $c$), with $2Md$ units. The
last layer $L$ sums each pair, so $L\,\xi(h_g)=e_g\otimes v_g$ and
$z=b+\sum_ge_g\otimes v_g$ holds every payload in its own block.

\emph{Read.} The first layer of the decoder computes the $2M$ gate units from the
target code (through $Bh_j$) and passes the $Md$ coordinates of $z-b$ through
$2Md$ units (through $Az$). The second hidden layer applies the product identity
to every pair $(m,c)$ with the target's gates, and the output layer sums over
$m$, which returns $v_{g(j)}$. Widths are $O(M+Md)$ on both sides, and the codes
have dimension $p=O(\log M)$.

\emph{(ii) Cost.} We use invariance of domain: there is no continuous injective
map from a nonempty open subset of $\R^N$ to $\R^{N'}$ with $N'<N$.

The target state does not depend on the payloads, so every target reads its
payload from $z$ alone. Since each group has its own target, $z$ determines all
payloads, and the map $(v_1,\ldots,v_M)\mapsto z$ is injective on $U^M$. It is
continuous because $\xi$ is. Hence $D\geq Md$. Moreover $z-b$ lies in the range
of $L$, a subspace of dimension at most $\operatorname{rank}L\leq W$. The same
argument gives $W\geq Md$, so $L$ has $DW\geq(Md)^2$ weights.

On the read side, the output of target $j$ depends on $z$ only through $Az$,
because $Bh_j+c$ does not depend on the payloads. Each target recovers its own
payload from $Az$, and every group has a target, so all payloads are determined
by $Az$. The map $(v_1,\ldots,v_M)\mapsto Az$ is therefore continuous and
injective on $U^M$, which gives $W'\geq Md$ and at least $W'D\geq(Md)^2$ weights
in $A$.

For the multiplicative count, the learned address table contributes $O(Mp)$
parameters; shared node/address projections contribute $O(hp+p^2)$; and the
value/output maps contribute $O(hd+d^2)$. None of the dense shared maps has a
width proportional to the product $Md$. \qed

\section{Capacity Proofs}
\label{app:capacity}

\subsection{Proof of Theorem~\ref{thm:lower}}

For each permutation \(\pi\in S_n\), let \(T_\pi\) be the full
source-dependent transcript available to the targets. If
\(\pi\neq\pi'\) but \(T_\pi=T_{\pi'}\), fixed target identifiers and a
deterministic decoder produce the same joint output on both instances, which
contradicts exact recovery. Therefore \(\pi\mapsto T_\pi\) must be injective,
so the transcript must take at least \(n!\) distinct values.

With \(L_{\mathrm{eff}}\) independently usable states, \(d\) coordinates per
state, and at most \(2^b\) values per coordinate, the transcript has at most
\(2^{L_{\mathrm{eff}}db}\) possible values. Hence
\(2^{L_{\mathrm{eff}}db}\geq n!\), which proves Eq.~\eqref{eq:capacity}.
Stirling's approximation gives \(\log_2(n!)=n\log_2n-O(n)\). With several global
layers, \(L_{\mathrm{eff}}db\) should be replaced by the capacity of the complete
source-dependent transcript. For a multiplicative memory whose $M$ rows can vary independently and are separately selectable, \(L_{\mathrm{eff}}=M\).
Average-case bounds would require a rate--distortion argument.

\subsection{Why symmetric virtual nodes do not factorize memory}

\begin{lemma}[Symmetric VN clones do not factorize memory]
\label{lem:vn-symmetry}
Consider \(M\) VNs with identical initial states, identical neighborhoods, and
shared permutation-equivariant update functions.  Their states remain identical
at every layer.
\end{lemma}

\begin{proof}
The property holds at initialization.  If all VNs are identical at layer
\(\ell\), they receive identical neighbor multisets and aggregated messages.
The shared update therefore produces identical states at layer \(\ell+1\).
\end{proof}

Their joint state always has the form \((z,\ldots,z)\), so it has no more
reachable values than one VN and contributes only one effective compartment.

\subsection{Proof of Theorem~\ref{thm:upper}}

Let \(k=\lceil n/M\rceil\) and fix a public partition
\(I_1,\ldots,I_M\) with \(|I_m|\leq k\).  Let \(g(i)\) and \(r(i)\) denote
the group and within-group position of identifier \(i\).
A label in \([n]\) uses \(q=\lceil\log_2 n/b\rceil\) coordinates of \(b\) bits.
Row \(m\) stores its \(\lceil\log_2 M/b\rceil\)-coordinate group address and
\(k\) label blocks of length \(q\). Source \(i\) uses the address of row \(g(i)\)
and a payload \(v_i\) that contains the code of \(\pi(i)\) in block \(r(i)\) and
zeros elsewhere. In the hard-routing limit of Theorem~\ref{thm:addressing}
(\(\kappa\to0\), or exactly with a thresholded map), a row-normalized memory
stores in row \(m\) the average of its group's payloads, and a jointly
normalized memory returns the same average to a target of group \(m\). Blocks do
not overlap within a group, so each block holds one label code scaled by
\(1/|I_m|\), a public constant that the decoder inverts. Target \(j\) reads row
\(g(j)\) and block \(r(j)\), recovering \(\pi(j)\). Source order does not matter
and permuting targets only permutes their reads, which proves equivariance and
Eq.~\eqref{eq:upper}.

For labels independently drawn from an alphabet of size \(C\), the analogous
capacity requirement is \(L_{\mathrm{eff}}db\geq n\log_2 C\), and the
construction uses \(O((n/M)\log C/b)\) label coordinates per row.

\section{Support-Mass Proofs}
\label{app:mass}

\paragraph{Proposition~\ref{prop:replication}.}
Under uniform replication, every source weight and value appears $r$ times.
For a row-normalized read, both the weighted value sum and the row normalizer
scale by $r$, so the normalized row content is unchanged. For a jointly
normalized read, numerator and denominator also scale by the same factor.
Hence the exposed normalized read is unchanged while $Z$ scales by $r$. An
unnormalized read scales with $r$. \qed

\paragraph{Proposition~\ref{prop:anchor-bounded}.}
The weights \(w_0/(w_0+Z)\) and \(w_i/(w_0+Z)\) are nonnegative and sum to one,
so \(z(X)\) lies in the convex hull of \(\{a\}\cup\{v_i\}\) and its norm is at
most \(\max(\|a\|,\max_i\|v_i\|)\). The log-odds equal
\(\log(Z/w_0)=\log Z-s_0\), strictly increasing in \(Z\), and
\(Z(rX)=rZ(X)>Z(X)\) for \(r>1\). \qed

\paragraph{Theorem~\ref{thm:wlinjective}.}
Let the alphabet be \(\{1,\ldots,C\}\), use one-hot values \(v_i=e_{c(i)}\), zero
source logits, and an anchor with zero logit and value \(e_0\) on a reserved
coordinate. Then \(z_0=1/(1+|X|)\) and \(z_c=n_c/(1+|X|)\), where \(n_c\) is the
number of elements of color \(c\). The map \(X\mapsto z\) is injective: \(|X|\)
follows from \(z_0\), and each \(n_c=z_c/z_0\). A sufficiently expressive map
applied to \(z\) and to the center color then implements one 1-WL refinement on
neighborhoods of bounded size. For a normalized memory without a separate
cardinality path, Proposition~\ref{prop:replication} gives identical reads on
\(X\) and \(rX\), so it cannot be injective on a domain containing both. \qed

\paragraph{Scope of Proposition~\ref{prop:replication}.}
The proposition does not say that a system can never retain $Z$. A
linear-attention implementation may keep its denominator accumulator, and an
architecture that exports it already has a mass path. The question is how to keep
a bounded normalized read while making the support available through the same
interface. For a jointly normalized memory, the reference is implemented by
adding a positive vector $c_0$ to the denominator and $c_0\odot a$ to the
numerator; if the raw denominator is already exposed to the decoder, this is
unnecessary for observability. With $a=0$, the anchored read is
$z(X)=\frac{Z}{e^{s_0}+Z}\,\bar v$: the normalized content is multiplied by a
support gate, and $s_0$ sets the scale at which each row or head becomes active.

\paragraph{Relation to Softmax$_1$.}
With zero anchor value, \(z_h(X)=\frac{Z_h}{e^{s_{0,h}}+Z_h}\bar v_h\). Setting
\(s_{0,h}=0\) gives Softmax$_1$ exactly. The fixed \(+1\) already breaks
replication invariance; the learned anchor calibrates the reference separately
for each row and head, and the log-odds path exposes it explicitly.

\section{Permutation-Invariant Semantic Routing Protocol}
\label{app:semantic-routing}

\paragraph{Episodes.}
Each episode contains $n=8$ source key--payload pairs and eight target queries.
New base keys are sampled for every graph as normalized random vectors in
$\R^8$. A source key and its matching target query are independent Gaussian
perturbations of the same base key followed by $\ell_2$ normalization. Payloads
are sampled independently from eight classes. Source and target orders are
independently permuted. Train/IID key noise is $0.07$ and the main OOD split is
$0.16$.

\paragraph{Best-tuned $192$-scalar comparison.}
All models contain exactly $192$ dynamic global-memory scalars and are trained
for $50{,}000$ updates with batch size $192$. The slot model uses $32$ slots of
width $6$, hidden width $40$, temperature $0.14$, and a semantic projection
shared between source keys and target queries. Its learning rate is stepped
from $3\times10^{-2}$ to $10^{-2}$, $3\times10^{-3}$, and $10^{-3}$ over
training. It contains $2{,}094$ learned parameters.

For the parameter-matched VN, the AdamW learning rate is
$1.5\times10^{-3}$ through $10$k updates, $7.5\times10^{-4}$ through $30$k, and
$3\times10^{-4}$ thereafter. Each family is tuned separately; final evaluation
uses seeds $100$--$102$.

\paragraph{Approximate forward FLOPs.}
A multiply plus an add counts as two FLOPs; biases, activations, normalization,
and softmax are omitted. The parameter-matched VN uses $0.0333$ MFLOPs, the
high-capacity VN $0.934$ MFLOPs, and the $32\times6$ slot model $0.0420$ MFLOPs per
graph under this convention.

\paragraph{Post-hoc noise robustness.}
For representative seed-$100$ models frozen after $50$k updates, target
accuracy (\%) is:
\begin{center}
\begin{tabular}{@{}crr@{}}
\toprule
Noise $\sigma$ & High-capacity VN & Slots\\
\midrule
$0.07$ & $94.6$ & $96.3$\\
$0.16$ & $75.7$ & $84.8$\\
$0.20$ & $66.6$ & $77.0$\\
$0.24$ & $59.4$ & $69.1$\\
$0.30$ & $51.4$ & $58.4$\\
\bottomrule
\end{tabular}
\end{center}
This is a benchmark-specific robustness pattern, not a general OOD guarantee.

\paragraph{Permutation diagnostic.}
Reordering only the source set changes slot-model logits by at most
$1.9\times10^{-6}$, consistent with source-order invariance up to floating-point
error.

\section{Long-Budget Write/Read Ablation}
\label{app:write-read-ablation}

We rerun the identifier-based Two-Radius lookup with online data for $5{,}000$
updates and three seeds. The slot memory has $16\times12=192$ dynamic scalars.
The \textsc{No routed write} variant gives every slot the same pooled
source-dependent summary but keeps target-to-slot attention. The
\textsc{Implicit read} variant keeps content-dependent routed writes, flattens
the whole slot memory, and uses a target-conditioned MLP decoder. The
\textsc{Full slots} model uses routed writes and target-to-slot reads.

\begin{table}[h]
\centering
\small
\caption{Long-budget Two-Radius write/read ablation. Target accuracy (\%) at
fixed update budgets; mean $\pm$ s.d. over three seeds.}
\label{tab:write-read-ablation}
\begin{tabular}{@{}lrrrrr@{}}
\toprule
Model & Params & 500 & 1k & 2k & 5k\\
\midrule
No routed write & 3.85k & $10.7\pm0.5$ & $23.4\pm2.3$ & $49.0\pm12.3$ & $58.6\pm13.5$\\
Routed write, implicit read & 32.2k & $79.2\pm2.7$ & $93.1\pm12.0$ & $93.9\pm10.6$ & $94.3\pm9.8$\\
Full routed write/read & 3.85k & $100.0$ & $100.0$ & $100.0$ & $100.0$\\
\bottomrule
\end{tabular}
\end{table}

At $5$k updates, two of three implicit-read seeds achieve exact graph recovery
and one remains near $83\%$ target accuracy, giving
$66.7\pm57.7\%$ exact-graph accuracy. The full model is exact on every seed by
$500$ updates. These results refine the short-budget ablation: a sufficiently
expressive decoder can recover routed contents without an explicit slot read,
so target-to-slot attention is not required for representation. Under this
training protocol, however, the explicit read is substantially more
parameter-efficient and consistent across seeds. Removing routed writes is more
damaging because the source-dependent part of the memory no longer forms
separate compartments.

\section{Multiplicity-Aware Two-Radius Protocol}
\label{app:protocol}

\subsection{Data and batching}

Table~\ref{tab:counted} uses \(n=12\) identifiers, \(12\) labels, one
structural central node, maximum base multiplicity \(A=4\), and
replication scales \(r\in\{1,2,3\}\). For every base example, the labels
form a uniformly sampled permutation \(\pi\in S_{12}\), and the base
multiplicities are sampled independently as
\(a_i\sim\operatorname{Unif}\{1,2,3,4\}\).
The three paired graphs use counts \(c_i=ra_i\), so target counts lie in
\(\{1,\ldots,12\}\). Source order is independently shuffled in every rendered
graph, and target identifiers are placed in a random order.

The directed graph has edges from every source to the central node and from
the central node to every target, with mean-normalized aggregation and no
reverse edges. This prevents the fixed target set from entering the source
summary through the local path. The largest graph contains \(144\) sources,
one central node, and \(12\) targets. Smaller graphs are padded to \(157\)
positions; node and source masks remove padding from message passing, pooling,
and attention.

Data are generated online. Each training minibatch samples \(32\) base
assignments and renders each at all three scales, giving an effective batch of
\(96\) graphs. Over \(10{,}000\) updates, training processes \(320{,}000\) base
assignments, or \(960{,}000\) rendered graphs. Each validation epoch uses \(8\)
fresh minibatches, corresponding to \(256\) base assignments and \(768\)
rendered graphs.

\subsection{Model variants}

All variants use \(128\)-dimensional identifier, label, and role embeddings,
followed by a two-layer input MLP, LayerNorm, and three residual
mean-aggregation MPNN layers. The label and count heads are two-layer MLPs.
The count head has \(12\) classes, where class \(c_i-1\) represents count
\(c_i\).

The \textsc{MPNN} baseline has no auxiliary global state. The additive
\textsc{VN} mean-pools the sources with a mask, updates one learned virtual
state with a three-layer MLP, and broadcasts it to all targets through a
residual target MLP with scale \(0.2\). 
Both cross-attention variants use \(M=12\) learned static slot addresses and
\(4\) heads. Identifier embeddings are projected into a shared address space
used both for source-to-slot writes and target-to-slot reads; dynamic slot
content is not reused as an address. The write attends only to source tokens,
and both routing directions use temperature \(0.35\). Slot contents pass
through an output projection, LayerNorm, a residual feed-forward block, and a
second LayerNorm. The retrieved target context is integrated by FiLM with
\(\gamma=0.5\tanh(\widehat\gamma)\) and \(\beta=\tanh(\widehat\beta)\), and the
final FiLM linear map is initialized with weight standard deviation
\(10^{-3}\).

The anchored variant appends one private, learnable anchor logit to each slot
and head before the source softmax. Anchor logits start at zero, so the model
initially matches the unit reference mass of Softmax$_1$ but can learn a
different support scale for every slot and head. The per-head log-odds
\(\log(1-\alpha_0)-\log\alpha_0\) pass through a two-layer mass MLP and are
added to the slot content after content normalization.

\subsection{Paired replication protocol}

For a fixed base assignment \((\pi,a)\), the three rendered graphs differ only
by the common scale \(r\). Copies of the same identifier have identical
identifier, label, role, neighborhood, and address features. Graph cardinality
is never given as an input. The source-to-slot attention set contains only
source tokens: central and target nodes are excluded because their fixed,
non-replicated presence would act as an implicit reference. A normalized
source read is therefore exactly invariant across the paired scales, while the
correct count labels change.

\subsection{Optimization and model selection}

All models are trained with AdamW for \(200\) epochs and \(50\) minibatches per
epoch. The learning rate is \(10^{-4}\) for embeddings, the local backbone, and
prediction heads, and \(2\times10^{-4}\) for the global block. Weight decay and
dropout are zero, and gradients are clipped to norm \(5\). The loss is
\(\mathcal L=\mathcal L_{\mathrm{label}}+w_{\mathrm{count}}\mathcal
L_{\mathrm{count}}\), with \(w_{\mathrm{count}}=0\) through epoch \(30\),
increased linearly to \(0.5\) over epochs \(31\)--\(70\), and kept at \(0.5\)
afterwards. At each epoch, metrics are averaged over \(8\) fresh validation
minibatches, and we report the epoch with the best validation \textsc{Both}.
The result in Table~\ref{tab:counted} uses seed $0$.

\subsection{Diagnostics}

For one base assignment rendered at all three scales, we compare the full
global states with the \(r=1\) state. A normalized memory should show zero
maximum and mean difference up to floating-point error, while the anchored
state should change with \(r\). We also record write weights, anchor weights,
read weights, count accuracy at each scale, and exact graph recovery.

\section{Tree-NeighborsMatch Protocol}
\label{app:tree}

A complete binary tree of depth $r$ has $2^r$ candidate leaves. Each candidate
carries a label and a distinct local neighbor-count feature; the root carries
one such count and must output the corresponding candidate label. As in the
controlled benchmark setup, the count is provided directly as an input feature
so that the experiment tests communication rather than learning to count.
Episodes are generated online. The MPNN baseline
uses $r+1$ local message-passing layers. Global-memory models use the same task
generator and are trained for $50{,}000$ updates. The slot models use $M=64$
memory rows of width $16$. Table~\ref{tab:tree-neighbors} reports multi-seed estimates at all tested
depths. Parameter counts in the main table are those at $r=5$.

\paragraph{OGB-style virtual node.}
The OGB-style VN keeps a learned state per graph, initialized at zero. Before
every local layer the state is added to every node; after every layer it is
updated as $\mathrm{vn}\leftarrow\mathrm{LN}(\mathrm{vn}+\MLP(\mathrm{pool}(h)+\mathrm{vn}))$
with a sum pool, as in the OGB baselines \citep{hu2020open}. Its memory width is
the node width. In additional single-seed diagnostic runs at $r=4$ (width $64$, $4{,}000$
updates, $18.8\%$), doubling the width to $128$ ($19.0\%$) or tripling the
training length ($19.2\%$) did not materially change accuracy. These checks do not exhaust the tuning space; within the tested range, increasing
width or training length alone did not materially improve the result.

\paragraph{Depth $r=6$.}
At $r=6$ ($64$ candidates), the table reports multi-seed estimates. Slots reach
$26.2\pm0.2\%$ and anchored slots $25.7\pm0.2\%$, compared with
$5.1$--$5.7\%$ for the local and pooled baselines and $1.6\%$ chance. The
anchored bilinear realization reaches $100\%$ at $r=3,4,5$ and a mean accuracy
of $20.9\%$ at $r=6$. These results extend the same addressability pattern to a
larger candidate set, although the task is no longer solved at this depth.

\section{Standard Benchmarks (LRGB)}
\label{app:lrgb}

\paragraph{Protocol.}
We use Peptides-func (graph classification, macro AP), Peptides-struct (graph
regression, MAE) and PascalVOC-SP (node classification, macro F1) from LRGB
\citep{dwivedi2022long}, with OGB atom and bond encoders on Peptides and
channel-wise feature normalization on PascalVOC-SP. The backbone is GINE with a
skip connection around each layer and a multi-layer prediction head, following
the retuned protocol of \citet{tonshoff2023where}. The epoch is selected on validation and we report the mean test score over five
seeds. The final configuration adds one light anchored slot block, without
persistence, on the last layer of a GINE that already uses the OGB-style virtual
node. We use six layers, backbone width $128$, $32$ slots of width $32$, batch
size $128$, dropout $0.1$ (up to $0.5$ for the OGB-VN baseline), and weight
decay $0.01$.

\begin{table}[h]
\centering
\small
\caption{Standard LRGB benchmarks, mean test score over five seeds at the
best validation epoch.}
\label{tab:app-lrgb}
\begin{tabular}{@{}lrrr@{}}
\toprule
Model & Pep-func AP $\uparrow$ & Pep-struct MAE $\downarrow$ & VOC-SP F1 $\uparrow$\\
\midrule
GINE (no virtual node) & 0.664 & 0.261 & 0.378\\
GINE + pooled mean VN & 0.677 & 0.252 & 0.389\\
GINE + OGB-style VN & 0.711 & 0.239 & 0.402 \\
GINE + anchored slots & 0.701 & 0.233 & 0.391\\
GINE + OGB-style VN + anchored slots  & 0.712 & 0.230 & 0.404\\
GINE + anchored bilinear memory & 0.680 & 0.247 & 0.389\\
\bottomrule
\end{tabular}
\end{table}

Among the tested slot configurations, the light last-layer block gave the most
consistent validation results. The anchored bilinear realization gives lower mean scores on these three
tasks than the slot/OGB-VN variants, although both
multiplicative realizations solve the controlled addressability benchmarks.
This table therefore serves as a transfer check rather than a broad performance
claim. The addressable cross-attentive ablations were particularly prone to overfitting in our experiments, more tailored optimization could help improve the results.


\end{document}